\documentclass{article}

\usepackage{microtype}
\usepackage{graphicx}
\usepackage{subfigure}
\usepackage{booktabs} 
\usepackage{comment}

\usepackage{enumitem}

\usepackage{hyperref}

\usepackage[accepted]{icml2025}

\usepackage{amsmath}
\usepackage{amssymb}
\usepackage{mathtools}
\usepackage{amsthm}
\usepackage{multicol}
\usepackage{multirow}

\newcommand{\diff}{\mathrm{d}} 

 \newcommand{\MSE}{{\text{MSE}}}
 \newcommand{\tol}{{\text{tol}}}

 \newcommand{\xt}{\mathbf{x}_t}
 
 \newcommand{\z}{\mathbf{z}}

 \newcommand{\y}{\mathbf{y}}
 \newcommand{\x}{\mathbf{x}}
  \newcommand{\n}{\mathbf{n}}
  \newcommand{\bfa}{\mathbf{a}}
  \newcommand{\bfb}{\mathbf{b}}
  \newcommand{\bfc}{\mathbf{c}}

 \newcommand{\w}{\mathbf{w}}
 \newcommand{\eps}{\mathbf{\epsilon}}

 \newcommand{\bE}{\mathbb{E}}
\newcommand{\bN}{\mathbb{N}}
\newcommand{\bP}{\mathbb{P}}

\DeclareMathOperator{\Cov}{Cov}
\DeclareMathOperator{\tr}{tr}

\DeclareMathOperator{\Var}{Var}

 \newtheorem{prop}{Proposition}

\icmltitlerunning{Controlled and Constrained Sampling with Diffusion Models via Initial Noise Perturbation}

\begin{document}

\twocolumn[
\icmltitle{CCS: Controllable and Constrained Sampling with Diffusion Models via Initial Noise Perturbation}


\icmlsetsymbol{equal}{*}

\begin{icmlauthorlist}
\icmlauthor{Bowen Song}{yyy}
\icmlauthor{Zecheng Zhang}{comp}
\icmlauthor{Zhaoxu Luo}{yyy}
\icmlauthor{Jason Hu}{yyy}
\icmlauthor{Wei Yuan}{xxx}
\icmlauthor{Jing Jia}{zzz}
\icmlauthor{Zhengxu Tang}{yyy}
\icmlauthor{Guanyang Wang}{equal,xxx}
\icmlauthor{Liyue Shen}{equal,yyy}
\end{icmlauthorlist}

\icmlaffiliation{yyy}{Department of EECS, University of Michigan}
\icmlaffiliation{comp}{Kumo.AI}
\icmlaffiliation{xxx}{Department of Statistics, Rutgers University}
\icmlaffiliation{zzz}{Department of Computer Science, Rutgers University}

\icmlcorrespondingauthor{Liyue Shen}{liyues@umich.edu}
\icmlcorrespondingauthor{Guanyang Wang}{guanyang.wang@rutgers.edu}
\icmlcorrespondingauthor{Bowen Song}{bowenbw@umich.edu}
\icmlkeywords{Machine Learning, ICML}

\vskip 0.3in
]



\printAffiliationsAndNotice{\icmlEqualContribution}


\begin{abstract}

Diffusion models have emerged as powerful tools for generative tasks, producing high-quality outputs across diverse domains. 
However, how the generated data responds to the initial noise perturbation in diffusion models remains under-explored, which hinders understanding the controllability of the sampling process. 
In this work, we first observe an interesting phenomenon: the relationship between the change of generation outputs and the scale of initial noise perturbation is highly linear through the diffusion ODE sampling.
Then we provide both theoretical and empirical study to justify this linearity property of this input-output (\textit{noise-generation data}) relationship. 
Inspired by these new insights, we propose a novel \textbf{C}ontrollable and \textbf{C}onstrained \textbf{S}ampling method (\textbf{CCS}) together with a new controller algorithm for diffusion models to sample with desired statistical properties while preserving good sample quality. 
We perform extensive experiments to compare our proposed sampling approach with other methods on both sampling controllability and sampled data quality. 
Results show that our CCS method achieves more precisely controlled sampling while maintaining superior sample quality and diversity. 
\end{abstract}


\section{Introduction}
\begin{figure*}
    \centering
    \includegraphics[width=0.82\linewidth]{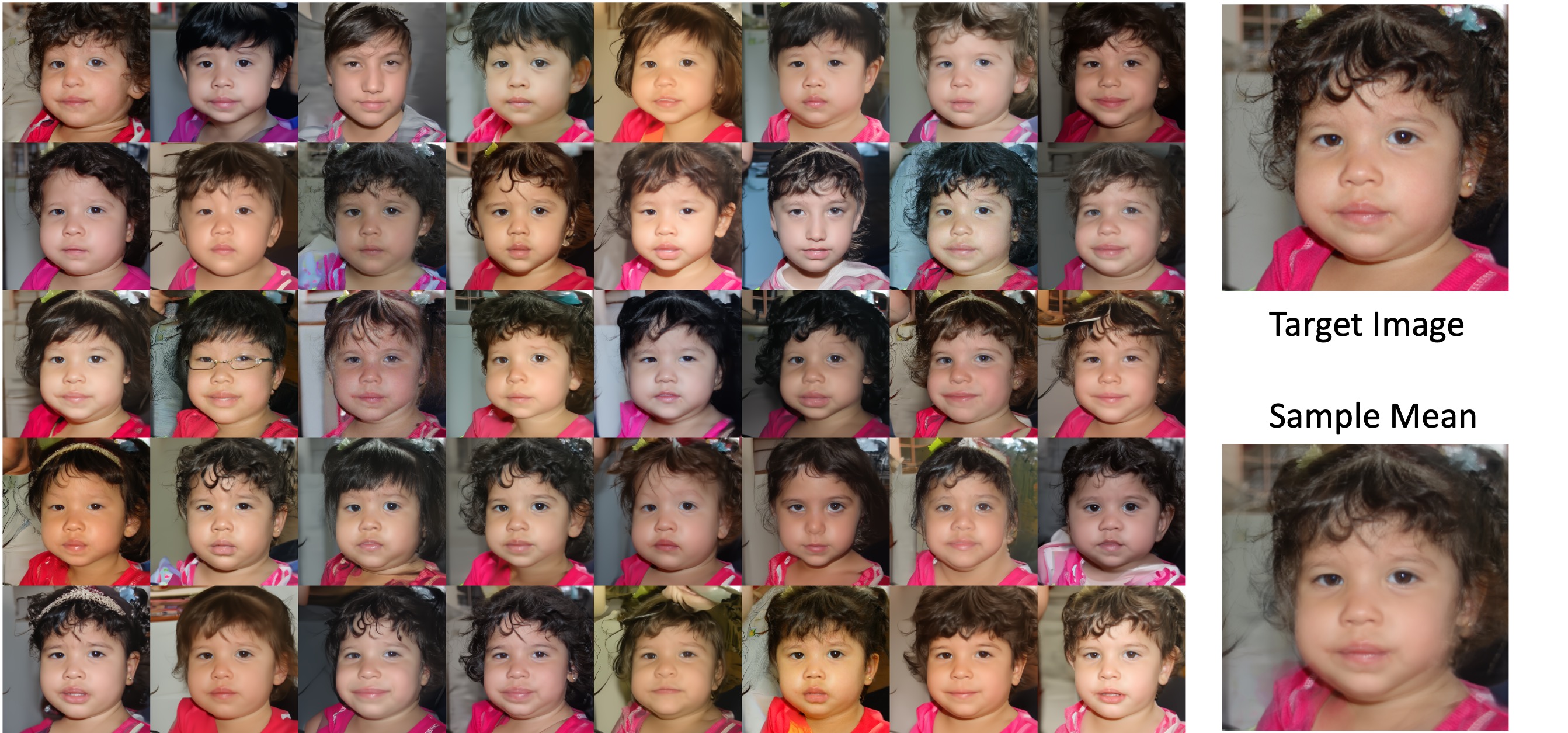}
    \vspace{-8pt}
    \caption{CCS Sampling with a target mean in the FFHQ validation dataset with $C_0 = 0.4$}
    \label{fig:CCS_FFHQ}
    \vspace{-16pt}
\end{figure*}
Recently, diffusion models achieve remarkable success in generative tasks such as text-to-image generation, audio synthesis~\cite{audio1, ldm}, as well as conditional generation tasks including inverse problem solving, image or video restoration, image editing, and translation~\cite{chen2024exploring, edit1, edit2, dps, resample, solve1, solve2, solve3, solve4, controlnet}. 
Despite these success, real-world scientific and engineering problems pose more challenges on requesting reliable and controllable generation as well as data privacy.

To tackle this, one important question is: \textit{How to control the distribution of samples from a diffusion model to match a specific target?} 
Previous works on controllable generation with diffusion models mostly focus on constraining the generation process sample-by-sample using either plug-and-play approaches~\cite{constrainedtts, solve1, solve2, solve3} or modifying the unconditional score~\cite{ldm, controlnet, control2, dps}, so that each sample can satisfy a measurement constraint. 
However, most prior works focus on per-sample control, with limited exploration of how to regulate the overall distribution of generated samples to meet specific statistical constraints, which is a crucial requirement in differential privacy~\cite{differentialprivacy}.
This inspires the novel task for controllable and constrained sampling we are targeting in this paper.
Considering the unique mechanism in diffusion sampling, we are motivated to exploit the initial noise control by studying this key question:
\textit{How do the initial noise perturbations affect the generated samples in diffusion models?}
Previous works~\cite{loco, subspacecluster} suggest that the learned posterior mean predictor function is locally linear with perturbation among a certain range of timesteps for diffusion models. 
However, this linearity cannot be applied to every timestep nor to the samples of diffusion models. 
From a new perspective, this work sheds lights on the relationship between input noise perturbations and generation data in diffusion models, by proposing a training-free approach.

First of all, we observe an interesting phenomenon that when using denoising diffusion implicit models (DDIM) sampling, the initial noise has a highly linear effect on the generation data at small or moderate scales.
Motivated by this observation, our study tries to justify this linearity property via initial noise perturbation theoretically and empirically. 

Based on the spherical interpolation to perturb the initial noise vector, we propose a novel controllable and constrained sampling method (CCS) together with a new controller algorithm for diffusion models to sample with desired statistical properties while preserving high quality and adjustable diversity.

Furthermore, we conduct extensive experiments to validate the linearity phenomenon and then investigate the controllability performance of our proposed CCS method by generating images centered around a specified target mean image with a certain distance.
Results demonstrate the superiority of our CCS method in both controllability and sampled image quality compared with baseline methods. 
Moreover, we show the potential of proposed CCS sampling for broader applications including image editing.

Our contributions can be summarized as below:
\begin{itemize}[noitemsep, topsep=0pt] 
    \item To the best of our knowledge, we for the first time investigate a novel problem of controllable and constrained diffusion sampling, given constraints on certain statistical properties while preserving high sample quality.
    \item Motivated by our new finding on the highly linear relationship between the initial noise and generated samples, we propose an innovative noise perturbation method with a controller algorithm for constrained sampling around a target mean with a specified distance, supported by solid theoretical and empirical justifications.
    \item Extensive experiments on three datasets with both pixel and latent diffusion models, validate our findings and theoretical results about the linearity phenomenon and proposed algorithm. 
    Results demonstrate the superior performance of our algorithm in achieving precise controllability within a constrained sampling framework.
\end{itemize}

\section{Background}

\subsection{Diffusion Models}

Diffusion models consists of a forward process that gradually adds noise to a clean image, and a reverse process that denoises the noisy images~\citep{song2020denoising, song2019generative}. 
The forward model is given by $\x_t = \x_{t-1} - 0.5\beta_t \Delta t\x_{t-1} + \sqrt{\beta_t} \Delta t \mathbf{\omega}$ where $\mathbf{\omega} \in \bN(0,I)$ and $\beta(t)$ is the noise schedule of the process.
The distribution of $\x_0$ is the clean data distribution, while the distribution of $\x_T$ is approximately a standard Gaussian distribution. When we set $\Delta t \rightarrow 0$, the forward model becomes $\diff\x_t = -0.5 \beta_t \x_t \diff t + \sqrt{\beta_t}\diff\omega_t$, which is a stochastic differential equation (SDE). 
The reverse of this SDE is given by:
\begin{equation}
    \diff x_t = \left( -\frac{\beta(t)}{2} - \beta(t) \nabla_{\x_t} \log p_t(\x_t) \right) dt + \sqrt{\beta(t)}d \overline{\mathbf{\omega}}.\nonumber
\end{equation}
One can training a neural network to learn the score function $\nabla_{\x_t} \log p_t(\x_t)$.
However, this formulation involves running many timesteps with high randomness. We can also compute the equivalent Ordinary Differential Equation (ODE) form to the SDE, which has the same marginal distribution of $p(\x_t)$. 
A sampling process, called denoising diffusion implicit models (DDIM), modifies the forward process to be non-markovian, so as to form a deterministic probability-flow ODE for the reverse process~\cite{song2021denoising}. 
In this way, we are able to achieve significant speed-up sampling. More discussion on this can be found in Section~\ref{sec: linear change perturbation}.

\subsection{Constrained Generation with Diffusion Models}

Constrained generation requires to sample $\x_0$ subject to certain conditions or measurements $\y$. The conditional score at $T$ can be computed by the Bayes rule, such that \begin{equation}
\label{cg}
    \nabla_{\mathbf{x}_t} \log p_t(\mathbf{x}_t | \mathbf{y}) = \nabla_{\mathbf{x}_t} \log p_t(\mathbf{x}_t) + \nabla_{\mathbf{x}_t} \log p_t(\mathbf{y} | \mathbf{x}_t).
\end{equation}
The second term can be computed through classifier guidance~\cite{beatGAN}, where an external classifier is trained for $p_0(\y|\x_0)$ or $p_t(\y|\x_t)$, and then can be plug into the diffusion model through Eq.~\ref{cg}. Diffusion posterior sampling~\cite{dps} further refines this formulation by proposing to perform posterior sampling with the approximation of $p(\y | \x_t) \approx p(\y | \hat{\x}_0)$, where $\hat{\x}_0$ is the Minimum Mean Square Error (MMSE) estimator of $\mathbf{x}_0$ based on $\x_t$.

Another line of works exploit hard consistency, which projects the intermediate noise to a measurement-consistent space during sampling via optimization and plug-and-play \cite{mcg, solve1, constrainedtts, resample}. However, the projection term can damage the sample quality~\cite{dps}. 
However, these works all target on controlling each individual sample.
To our best knowledge, few works explore how to control the distribution of generated samples to match certain statistical constraints, such as centered around a specified target mean with certain distance,
which is the target for this work.

\subsection{Noise Perturbation in Diffusion Models}

Noise adjustment for diffusion models has been explored in image editing, video generation, and other applications~\cite{edit1, edit2, styleid, noise0, noise1, noisediffusion} for changing the style or other properties of the generated data. 
However, a principled study on how the noise adjustment affects the samples is limited in diffusion models. 
Recently, \citet{loco, subspacecluster} observe the local linearity and low-rankness of the posterior mean predictor $\mathbf{\hat{x}}_0$ based on $\x_t$ in large timesteps, but this study cannot extend to the analysis of generated samples. 
In this work, we investigate how initial noise perturbations affect the samples generated from the diffusion model in the ODE sampling setting.

\begin{figure*}
    \centering
    \includegraphics[width=0.9\linewidth]{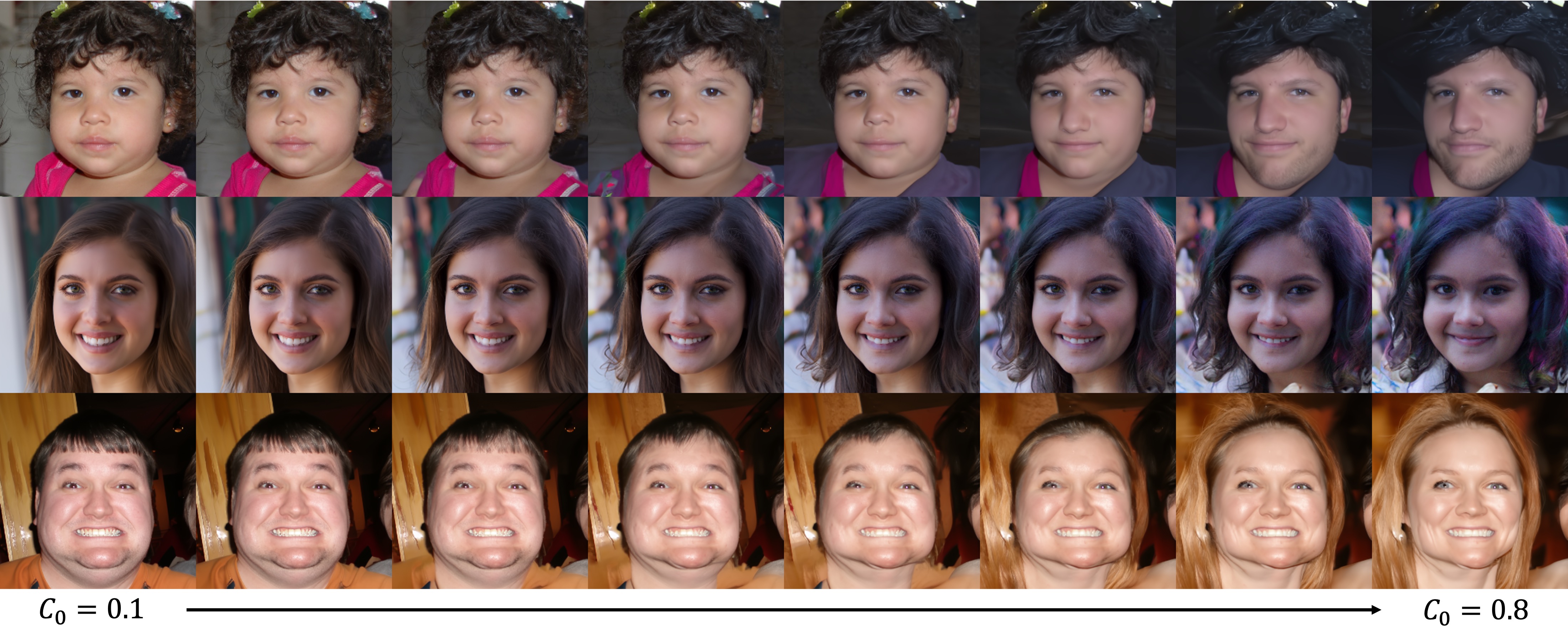}
    \caption{Qualitative demonstration of linearity when increasing scale of perturbation. For each target mean, we sample a perturbation noise and gradually increase $C_0$ (0.1 at a time) to increase the magnitude of the perturbation.}
    \label{fig:qualitativelinear}
\end{figure*}

\section{Influence from Initial Noise Perturbation}\label{sec: linear change perturbation}

This section analyzes how small perturbations in the input noise affect the generation data under the DDIM sampling framework. We show that a slight change in the initial noise leads to an approximately linear variation in the sampled images. This result is quantified from two perspectives: the discretized DDIM sampling process \cite{song2021denoising} and the associated continuous-time ODE. 
Our mathematical analysis relies on minimal assumptions, which also serves as the foundation for our proposed CCS algorithm in Section~\ref{sec: sampling with control}.
\vspace{-20px}
\subsection{Preliminary: DDIM Sampling}
Fix the total sampling timesteps $T$ and an initialization noise sample $\x_{T}$, \citet{song2021denoising} generates samples from the backward process $\x_T \rightarrow \x_{T-1} \rightarrow \ldots \rightarrow \x_0$ using the following recursive formula:
\vspace{-20px}
\begin{equation}\label{eqn: DDIM}
\begin{aligned}
    \x_{t-1} = \sqrt{\alpha_{t-1}} 
\left(
\frac{\x_t - \sqrt{1-\alpha_t} \, \epsilon_{\theta}^{(t)}(\x_t)}{\sqrt{\alpha_t}}
\right) \\
 + \sqrt{1-\alpha_{t-1} - \sigma_t^2} \, \epsilon_{\theta}^{(t)}(\x_t)
+ \sigma_t \epsilon_t,
\end{aligned}
\end{equation}
where $\alpha_t$ corresponds to the noise schedule in DDPM, $\epsilon_{\theta}^{(t)}(\x_t)$ is the predicted noise given by the pre-trained neural network with parameter $\theta$, $\epsilon_t$ is the standard Gaussian noise, and $\sigma_t$ is a hyperparameter. 
The DDIM sampler~\cite{song2021denoising} sets $\sigma_t = 0$ to make the backward process deterministic once $\x_T$ is fixed. It is known (\textit{e.g.}, eq (11) of \citet{dhariwal2021diffusion}) that predicting the noise is equivalent to predicting the score function up to a normalizing factor, \textit{i.e.}, $
\epsilon_\theta^{(t)}\left(\xt\right) \approx -\sqrt{1- \alpha_t}\nabla_{\xt}\log p_{t}(\xt)$. By setting $\sigma_t = 0$ and substituting $\epsilon_\theta^{(t)}$ with its corresponding estimand, we obtain the \textit{idealized DDIM process}:
\begin{equation}\label{eqn:idealized ddim}
\begin{aligned}
\x_{t-1} &= \sqrt{\alpha_{t-1}} 
\left(
\frac{\xt + (1-\alpha_t \,) \nabla_{\xt}\log p_{t}(\xt)}{\sqrt{\alpha_t}}
\right)
 \\
 & ~~~~- \sqrt{(1-\alpha_{t-1})(1-\alpha_t)} \, \nabla_{\xt}\log p_{t}(\xt).
\end{aligned}
\end{equation}

If we treat the index $t$ as a continuous variable (and rewrite $\alpha_t$ as $\alpha(t)$ to avoid confusion), it is known in eq (14) of \cite{song2020denoising} that DDIM is the Euler-discretization of the following (backward) ODE:
\begin{align*}
    \diff \bar{\x}_t = \epsilon_\theta^{(t)}\left(\frac{\bar{\x}_t}{\sqrt{\sigma^2(t) + 1}}\right)\diff \sigma(t),
\end{align*}
where $\bar{\x}_t = \x_t/\sqrt{\alpha(t)}, \sigma(t) := \sqrt{(1- \alpha(t))/\alpha(t)}.$ 
Thus, we can similarly write the idealized ODE as:
\begin{align}\label{eqn:idealized ODE}
    \diff \bar{\x}_t = -\sqrt{1-\alpha(t)}\nabla\log p_t\left(\frac{\bar{\x}_t}{\sqrt{\sigma^2(t) + 1}}\right)\diff \sigma(t).
\end{align}

We now examine how a small perturbation $\x_T \rightarrow \x_T + \lambda \Delta \x$ would affect the output sample at time $t = 0$ through both the discrete \eqref{eqn:idealized ddim} and continuous time \eqref{eqn:idealized ODE} perspectives.

\textbf{Related work:} Theorem 1 in \citet{chen2024exploring} presents a related result on the impact of initial noise perturbation. Our study differs from theirs in a variety of aspects. Firstly, they study $\bE[\x_0 \mid \x_t + \lambda \Delta\x] - \bE[\x_0 \mid \x_t] $ under the (stochastic) diffusion process.  In contrast, we directly examine the output \(\x_0\) given the initializations \(\x_t\) and \(\x_t + \lambda \Delta\x\) under the deterministic DDIM \eqref{eqn:idealized ddim} or the ODE process \eqref{eqn:idealized ODE}. Secondly, \cite{chen2024exploring} assumes that \(p_0\) is a low-rank mixture of Gaussian distributions, which allows for an analytical solution for \(p_t\). In contrast, our weaker assumptions render \(p_t\) analytically intractable. Consequently, we use very different techniques, such as ODE stability theory and Grönwall's inequality, to study the system's behavior.

\subsection{DDIM Discretization}\label{subsec:linear, DDIM}
We simplify notations and write the idealized DDIM \eqref{eqn:idealized ddim} as
\begin{align*}
\x_{t-1} := \eta_t \x_t + \lambda_t \nabla_{\xt}\log p_{t}(\xt),
\end{align*}
where $\eta_t = \sqrt{\alpha_{t-1}^{-1}\alpha_t}$ and $\lambda_t = \sqrt{\alpha_{t-1}^{-1}\alpha_t}(1-\alpha_t) - \sqrt{(1-\alpha_{t-1})(1-\alpha_t)}$.

Fix $T$, we are interested in studying $\x_0(\x_{T}, T)$ and $\x_0(\x_{T} + \lambda\Delta\x, T)$, where $\Delta\x$ is a unit direction and $\lambda$ is a (small) real number. 
The notation $\x_0(a,t)$ stands for the endpoint $\x_0$ by applying the idealized DDIM procedure with $\x_t = a$ at timestep $t$. We have:

\begin{prop}\label{prop: linearity, DDIM}
    With all the notations defined as above, assuming  $\log p_t$ is second-order differentiable for every $t\geq 1$, there exists a matrix-valued function $\gamma_0$ such that
    \begin{align*}
    \x_0(\x_T + \lambda\Delta\x, T) = \x_0(\x_T) + \lambda\gamma_0(\x_T)\Delta\x + o(\lambda). 
\end{align*}
In turn, 
\begin{align*}
    \lVert  \x_0(\x_T + \lambda\Delta\x, T) -  \x_0(\x_T, T)\rVert_2 \\
     = \lVert \lambda \gamma_0(\x_T)\Delta\x \rVert_2 + o(\lambda). 
\end{align*}
\end{prop}

Proposition \ref{prop: linearity, DDIM} shows that a linear perturbation of the input with magnitude $\lambda$ and direction $\Delta\x$ results in an approximately linear change in the output, with magnitude $\lvert \lambda \rvert \lVert \gamma_0(\x_T)\Delta\x \rVert_2$ and direction $\gamma_0(\x_T)\Delta\x$. Our assumption is based solely on the second-order smoothness of the score, which is weaker than most existing assumptions depending on the data  distribution $p_0$. For example, our assumptions hold under common conditions in the literature, such as the manifold hypothesis \cite{de2022convergence, song2019generative} or the mixture of (low-rank) Gaussian assumption \cite{gatmiry2024learning, chen2024exploring, chen2024learning}. \\
Furthermore, when at large $t$, $p_t$ is approximately Gaussian and $\nabla_{\x_t}\text{log}p_t(\x_t)$ is smooth, which leads to low linear approximation error. However, one might be concerned that the linear approximation error could grow significantly when \( t \) decreases and \( p_0 \) contains multiple clusters with low-density regions in between. Nevertheless, we now explain why this concern does not arise in practice. The coefficient $ f(t):= -\sqrt{\alpha_t}^{-1} \sqrt{\alpha_{t-1}(1 - \alpha_t)} + \sqrt{1 - \alpha_{t-1}}$ of $\epsilon_{\theta}^{(t)}(\mathbf{x}_t)$ in \eqref{eqn: DDIM}  is close to  $0$ for small $t$, as $\alpha_t\approx 1$. Moreover, the structure of the neural network \( \epsilon_\theta \)  ensures that the output is normalized and bounded in norm, so the change in output is also bounded. Consequently, for a small perturbation in \( \x_t \), we have \( \lVert f(t)(\epsilon_{\theta}^{(t)}(\x_t+\Delta x) - \epsilon_{\theta}^{(t)}(\x_t))\rVert_2 \approx 0 \) when \( t \) is small.

\subsection{ODE Stability}\label{subsec: linear, ODE}
Let $\bar \x_0(\x, T)$ be the solution of \eqref{eqn:idealized ODE} with initialization $\x_T = \x$ (i.e., $\bar\x_T = \x/\sqrt{\alpha(T)}$) at timestep $T$, and $x_0(\x, T) = \alpha(0)\bar\x_0(\x, T)$. With some technical assumptions that is detailed in Appendix \ref{subsec: proof of linear, ODE}, we have the following:
\begin{prop}\label{prop: linearity ODE}
    There exists a matrix-valued function $\psi_0$ such that:
    \begin{align*}
        \bar\x_0(\x_T +\lambda \Delta\x, T) = \bar\x_0(\x_T) + \lambda \psi_0(\x_T)\Delta\x + o(\lambda).
    \end{align*}
    In turn,
    \begin{align*}
        \x_0(\x_T +\lambda \Delta\x, T) = \x_0(\x_T) + \lambda\sqrt{\alpha(0)} \psi_0(\x_T)\Delta\x + o(\lambda).
    \end{align*}
\end{prop}

Proposition \ref{prop: linearity ODE} mirrors Proposition \ref{prop: linearity, DDIM} but is formulated in the continuous-time ODE setting. Its proof relies on ODE stability theory, showing that the output change is ``approximately linear" for sufficiently small \( \lambda \). Furthermore, under the same assumption, we establish that the change remains ``at most linear" for all \( \lambda \). The proof, which applies Grönwall's inequality, is provided in Appendix \ref{subsec: proof of linear, ODE}.

\begin{prop}\label{prop: at most linear ODE}
   With the same assumptions as above, there exists a constant $C(T)$ depending on $T$ such that for any $\lambda$:
    \begin{align*}
        \lVert \bar\x_0(\x_T +\lambda \Delta\x, T) - \bar\x_0(\x_T) \rVert_2 \leq C(T)\lvert \lambda \rvert \lVert \Delta \x \rVert_2.
    \end{align*}
\end{prop}

\section{Sampling with Control}\label{sec: sampling with control}
Now we discuss our controllable sampling algorithm. We preserve the notation $\x_0$ to denote a ``target image'' or ``target mean'' (\textit{e.g.}, the top right corner in Figure \ref{fig:CCS_FFHQ}). We also preserve the notation $\x_T := \text{DDIM}^{-1}(\x_0;0, T)$, the ``noise'' by running DDIM \eqref{eqn: DDIM} reversely from time $0$ to $T$. Our objective is to perturb $\x_T$ into a \textit{random} $\x_T'$ such that the generated image $\x_0'$ such that it has \textit{1. a sample mean close to} $\x_0$ while maintaining \textit{2. sufficient diversity and difference from the original image} and \textit{3. high image quality}. 
The closeness is quantified by $L$-2 norm distance $\rVert\bE[\x_0'] - \x_0\rVert_2$, and the diversity is measured by $\mathbb{E}[||\x_0' - \x_0||_2^2]$. A notable feature of our algorithm is that users can specify a desired level of diversity ($C_0$ in Fig.~\ref{fig:CCS_FFHQ}, \ref{fig:qualitativelinear}), and the generated images will match this level while ensuring \(\bE[\x_0'] \approx \x_0\). Our  mechanism is defined as $\x_T' = a \x_T + b\Delta$, where $\Delta$ is a random perturbation, and $a$ and $b$ are parameters to be specified shortly.

\subsection{Sampling around a Center}\label{subsec: sampling around center}
For an input of the form \(\x'_T = a\x_T + b\Delta\) with random $\Delta$, when \(b\) is small and \(a\) is close to 1, it can be regarded as a slight perturbation of \(\x_T\). Based on  Section \ref{sec: linear change perturbation}, the output will remain close to \(\x_0\) with an additional linear adjustment applied to \(b\n\). Thus, we define \(\hat{\x}_0' := \x_0 + bA\Delta\) as an approximation for \(\x_0'\), where \(A = \gamma_0(a\x_T + b\Delta)\) specified in Proposition \ref{prop: linearity, DDIM}. Since $\Delta$ is the only source of randomness in $\hat{\x}_0'$, we can easily calculate  $\bE[\hat \x_0']=\x_0 + bAE[\Delta]$ and $\Var[\hat \x_0']= b^2A\Cov(\Delta)A^\top$. We will now discuss the principles for our sampling design. 

 \textbf{High-quality image generation:} we first note that the input to both DDPM and DDIM samplers is standard Gaussian noise. The following feature is known as the ``concentration phenomenon'' of a high-dimensional Gaussian:
\begin{prop}\label{prop: gaussian concentration}
    Let $X\sim \bN(0, I_d)$, then for any $\delta \in (0,1)$
    \begin{align*}
        \bP\left[ \lVert X\rVert_2^2 \in (1 \pm \delta) d\right\rvert] \geq 1-  2 \exp\left(-\frac12d\left(\frac12 \delta^2 - \frac13 \delta^3\right) \right).
    \end{align*}   
\end{prop}
This result suggests that a standard Gaussian noise vector remains close to a hypersphere of radius \(\sqrt{d}\). For example, when \(d=50{,}000\) (a common dimension in imaging) and \(\delta=0.025\), Proposition~\ref{prop: gaussian concentration} guarantees with over 99.9\% probability that the squared norm is in the interval \(\bigl(0.975\,d,\;1.025\,d\bigr)\). Empirical results \cite{song2020denoising} further confirm that starting with a noise vector on this hypersphere is important for generating high-quality images.

Hence, we can expect $\lVert\x_T\rVert_2\approx \sqrt{d}$. Furthermore, we will design our mechanism to ensure \(\x'_T\) also has a norm $\approx \sqrt d$. 

\textbf{Close to target mean:} Our approximation $\hat\x_0'$ has expectation $\bE[\hat\x_0']  = \x_0 + bAE[\Delta]$. Thus, it is sufficient to select \(\Delta\) such that \(\mathbb{E}[\Delta] = 0\) in order to achieve:
$
\mathbb{E}[\hat{\x}_0] \approx \mathbb{E}[\hat{\x}_0'] = \x_0,
$
where the first approximation is justified by Proposition \ref{prop: linearity, DDIM} and \ref{prop: linearity ODE}, and is empirically validated in Figure \ref{fig:linear1}.

\subsection{Centering Feasibility}
The simplest strategy is to add a random noise vector $\Delta\x$ directly to \(\x_T\), expressed as \(\x_T' = \x_T + \Delta\x\) (with \(a = 1, b\Delta = \Delta\x\)). However, the following proposition demonstrates that this approach cannot produce high-quality images.

\begin{prop}\label{prop: centering feasibility}
  For any fixed vector $\x$, and any random vector $\Delta\x$ such that $\bE[\Delta\x] = 0$, the following holds:  
\[
\bE[\lVert \x + \Delta\x \rVert_2^2] =  \lVert \x \rVert^2_2 + \tr(\Cov[\Delta\x]) \geq \lVert \x \rVert_2^2,
\]
with equality if and only if $\Delta\x = 0$ almost surely.
\end{prop}

Proposition \ref{prop: centering feasibility} indicates that directly adding noise, $\x_T \rightarrow \x_T' := \x_T + \Delta\x$, pushes $\x_T'$ farther from the spherical surface. This partly explains why the average image becomes blurrier or noisier as the scale of $\Delta\x$ increases, since the drift term $\tr(\Cov[\Delta\x])$ grows larger, causing $\x_T'$ to deviate further from the sphere with radius $\lVert \x_T \rVert_2$.

This inspires us to consider the spherical linear interpolation method \cite{shoemake1985animating} for sampling, as described below. Similar approaches have been proposed by ~\citep{noisediffusion, song2020denoising}, but only for interpolating between two images.

\subsection{Spherical Interpolation}

Let vectors \( \bfa \) and \( \bfb \) satisfy \( \lVert \bfa \rVert_2 = \lVert \bfb \rVert_2 \) and form an angle \( \theta \). Then for any $\alpha\in(0,1)$, the vector obtained through spherical interpolation
$\bfc := \frac{\sin(\alpha\theta)}{\sin\theta} \bfa + \frac{\sin((1-\alpha) \theta)}{\sin(\theta)} \bfb$
satisfies $\lVert \bfc\rVert_2 = \lVert\bfa\rVert_2 = \lVert\bfb\rVert_2$.

In our case, for a standard $d$-dimensional normal noise vector $\eps$, it is known $\lVert\eps\rVert_2\approx \sqrt{d}\approx \lVert\x_T\rVert_2$. Therefore, we can do spherical interpolation between $\x_T$ and $\eps$ to obtain $\x_T'$. Our CCS algorithm is described in Algorithm \ref{alg: CCS Full inverstion}.

The perturbation mechanism corresponds to $\x_T' = a\x_T +b\Delta$ with $a = \sin(\theta - C_0)/\sin(\theta), b = \sin(C_0)/\sin(\theta), \Delta = \epsilon \sim \bN(0,I)$. 
$C_0 := \alpha\theta$ is defined as the parameter of perturbation scale.
This mechanism satisfies the design principles described in Section \ref{subsec: sampling around center}:  \(\bE[\eps] = 0\) ensures that the new sample remains close to the target mean, while the Gaussian concentration and spherical interpolation ensure that \(\lVert\x_T' \rVert_2 \approx \lVert \x_T\rVert_2\), resulting in high-quality generated images.

\begin{algorithm}
\caption{(Full Inversion) CCS Sampling}\label{alg: CCS Full inverstion}

\textbf{Requires: } target mean $\x_0$, perturbation scale $C_0$, number of diffusion model timesteps $T$

\textbf{Step 0:} Compute the DDIM inversion of $\x_0$, i.e. $\x_T = \text{DDIM}^{-1}(\x_0, 0, T)$

\textbf{Step 1:} Sample noise $\epsilon \sim \bN(0,I)$. Then compute \[\theta = \cos^{-1}\left(\frac{\epsilon \cdot \x_T}{||\epsilon||_2||\x_T||_2}\right)\]

\textbf{Step 2:} Compute $\x_T$ using spherical interpolation formula:
\[
\x_T' = \frac{\sin(C_0)}{\sin(\theta)} \cdot \epsilon + \frac{\sin(\theta - C_0)}{\sin(\theta)} \cdot \x_T
\]

\textbf{Step 3:} Output sample $\x_0' = \text{DDIM}(\x_T', T, 0)$
\end{algorithm}
Parameter \(C_0\) controls sampling diversity. In the extreme case  \(C_0 = 0\), we have \(\x_T' = \x_T\), so \(\x_0'\) matches \(\x_0\) exactly but has no diversity. A larger \(C_0\) makes the perturbed input deviate more from the original image and gets closer to  noise. This leads to greater diversity in the generated image.

Algorithm \ref{alg: CT} allows users to control the desired level of diversity. It works by calling Algorithm \ref{alg: CCS Full inverstion} for different values of \(C_0\), which are determined through binary search. Let The process is repeated until the desired diversity level (up to a small tolerance threshold) is reached: if the MSE of generated images to target mean is below target threshold, \(C_0\) is increased; otherwise, it is decreased.
\begin{algorithm}
\caption{Controller Tuning (CT)}\label{alg: CT}
\begin{algorithmic}[1]
\STATE \textbf{Input:} target mean $\x_0$, target diversity level $\MSE_{\text{target}}$, tolerance: $\tol$ 
\STATE \textbf{Initialize:} $C_{\text{low}} \gets 0$, $C_{\text{high}} \gets \pi/2$, $C_0 \gets \frac{C_{\text{low}} + C_{\text{high}}}{2}$

\WHILE{not converged}
    \STATE Sample a batch of  $\x_0'$ by Alg.~\ref{alg: CCS Full inverstion}  with $C_0$ and $\x_0$
    \IF{$ |\mathbb{E}[||\x_0' - \x_0||_2] - \MSE_{\text{target}} | < \tol$}
        \STATE \textbf{Break}
    \ELSIF{$\mathbb{E}[||\x_0' - \x_0||_2] > \MSE_{\text{target}}$}
        \STATE $(C_{\text{high}},C_0) \gets (C_0, \frac{C_0 + C_{\text{low}}}{2})$
    \ELSE
        \STATE $(C_{\text{low}}, C_0) \gets (C_0, \frac{C_0 + C_{\text{high}}}{2})$
    \ENDIF
\ENDWHILE
\end{algorithmic}
\end{algorithm}

The following theorem demonstrates that the CCS algorithm is able to precisely control the input distance.

\begin{prop}\label{prop: CCS control initial distance}
    Denote the dimensionality of $\x_T$ by $d$. Given an initial noise $\x_T$ with $\lVert \x_T\rVert_2 = (1+o(1))\sqrt{d}$, and fix a small $\delta > 0$. For any $M\leq (2-\delta)\sqrt{d}$, then we can find $C_0$ in Algorithm \ref{alg: CCS Full inverstion} such that with probability $p_d \rightarrow 1$ as $d\rightarrow\infty$, we have $\lVert \x_T' - \x_T\rVert_2 = M$.
\end{prop}

Since the dimensionality of our problem is sufficiently large, Proposition \ref{prop: CCS control initial distance} allows users to control \(M\) as the input distance. Consequently, Algorithm \ref{alg: CCS Full inverstion} can generate a random interpolants with an exact distance of \(M\) from the input. Furthermore, since the direction is uniformly distributed, and when $C_0$ is small, $\mathbb{E}[\x_T'] \approx \mathbb{E}[\x_T]$, and $\mathbb{E}[\x_0'] \approx \mathbb{E}[\x_0]$, which satisfies our design goal.
\vspace{-10px}
\subsection{Extension to Conditional Latent Diffusion Models}
Conditional diffusion models usually compute the conditional score with classifier-free guidance (CFG). Let $s_{\theta}(\x_t, t)$ be the predicted noise, it can be written in $s_{\theta}(\x_t, t) = s_{\theta}(\x_t, t, c_{null}) + \gamma (s_{\theta}(\x_t, t, c) - s_{\theta}(\x_t, t, c_{null}))$ where $\gamma$ is the CFG term, $c$ is the condition and $c_{null}$ is the null condition. Exact inversion is very challenging in a high CFG setting~\cite{exactinversion}, and reconstruction error in the autoencoder of latent diffusion models makes it even harder. Motivated by this, we propose a Partial-Inversion CCS Sampling algorithm (P-CCS). Instead of starting from the $T$, we pick an intermediate timestep $t_0$. Then, we compute the noise term from DDIM inversion by subtracting the clean component, sample a new noise from $\bN(0, (1 - \alpha_{t_0})I)$, and then perform spherical interpolation. Details of this Alg.~\ref{alg: CCS partial inverstion} can be found in the Appendix. Furthermore, we can sample around a \textit{edited} target mean by first performing DDIM inversion with source prompt, then apply P-CCS sampling, and finally run DDIM with target prompt. More details can be found in the experiments and the Appendix. 
\vspace{-10px}
\section{Experiments}

In the experiments, we aim to answer three questions: 1. Can we sample images that have a sample mean close to the target mean with a target MSE by our designed algorithms while maintaining good image quality? 2. Does the linearity phenomenon between the norm of residual images and the perturbation scale widely exist? 3. Can our proposed algorithm work in more challenging settings such as in conditional generation with CFG or image editing tasks?  
\vspace{-10px}
\subsection{Validation of Linearity Phenomenon}
\textbf{Experimental setting.}
We perform extensive experiments on both pixel diffusion models on the FFHQ and CIFAR-10 dataset and latent diffusion models on the Celeba-HQ and fMoW dataset. For each experiment, we first sample 50 images as target images from each validation dataset from FFHQ \cite{ffhq}, CIFAR-10\cite{cifar-10}, and Celeba-HQ \cite{celeba-hq}. We also pick one images each class from the validation set of the fMoW dataset \cite{fmow} for further verification. Then for the FFHQ and CIFAR-10 selected data, we use pixel diffusion models as backbone; for Celeba-HQ and fMoW we use stable diffusion 1.5 as the backbone. The prompt for Celeba-HQ is given by "A high quality photo of a face" and the prompt for fMoW is given by "satellite images". Then, we use each image as a target mean and perform CCS sampling as in Alg.\ref{alg: CCS Full inverstion}.

For each target image, we sample eight $C_0$ from a uniform [0, 0.9] distribution. For each $C_0$, we sample 24 images. Then we compute the average $L^2$ distance between the sampled images and the target mean for each scale. 

\textbf{Evaluations.}
To quantitatively evaluate the linearity phenomenon, we compute the R-square between the input perturbation scales and the normalized average residual norms (scale between 0-1) for 4 datasets with both pixel diffusion models and latent diffusion models. Note that since different target means can lead to different slopes by different Hessian matrices, we normalize the residual norms. Specifically, we compute empirical slope $a$ and bias $b$ between $x = \sin(C_0)$ and $y = \bE[\lVert \x_0' - \x_0\rVert_2]$ each target mean, and then normalize the average $L^2$ distance to be: $y' = \frac{y - b}{a}$.

\textbf{Results.}
We observe a very strong linearity in the above experiments. Especially for pixel diffusion models, the R-square exceeds 0.98 for both datasets, which indicates almost a perfect linear relationship. For latent diffusion models, the linearity is slightly weaker, but still above 0.94 in R-square for both datasets. This is expected since Stable Diffusion use a nonlinear autoencoder 
 and trained on a different dataset. We also present more quantitative results in Fig.~\ref{fig:quant_linear} and qualitative results in Fig.~\ref{fig:qualitativelinear}. Surprisingly, we also observe a very linear semantic change in additional to pixel-value change.
\begin{table}[ht!]
\centering
\begin{tabular}{|c|c|c|c|}
\hline
\multicolumn{2}{|c|}{Pixel Diffusion Models} & \multicolumn{2}{c|}{Latent Diffusion Models} \\ \hline
FFHQ & CIFAR-10 & CelebA-HQ  & fMoW \\ \hline
0.995 & 0.988 & 0.959 & 0.947 \\ \hline
\end{tabular}
\caption{R-square between scales of input perturbation and normalized residual norms}
\label{tab:diffusion_comparison}
\end{table}
\begin{figure}
    \centering
    \includegraphics[width=1.0\linewidth]{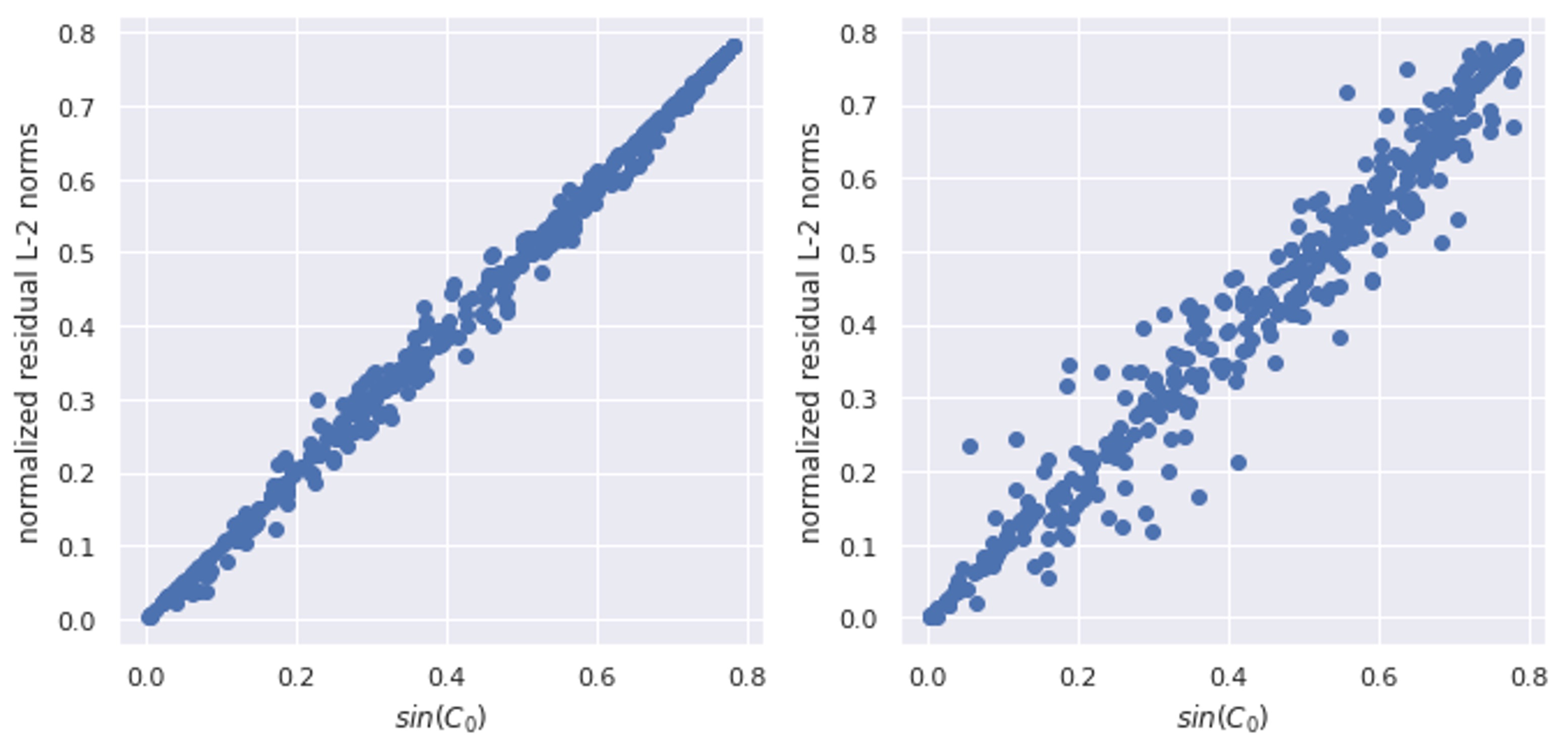}
    \caption{Quantitative demonstration of linearity when increasing scale of perturbation. With increased $\sin(C_0)$, the magnitude of perturbation increases, and the average $L^2$ distance between samples and the target image increases linearly. Left is the linearity on FFHQ dataset using pixel diffusion; Right is the linearity on Celeba-HQ dataset using Stable Diffusion 1.5.}
    \label{fig:quant_linear}
\end{figure}
\vspace{-10px}
\subsection{Controllable Sampling}

\begin{figure}
    \centering
    \includegraphics[width=1\linewidth]{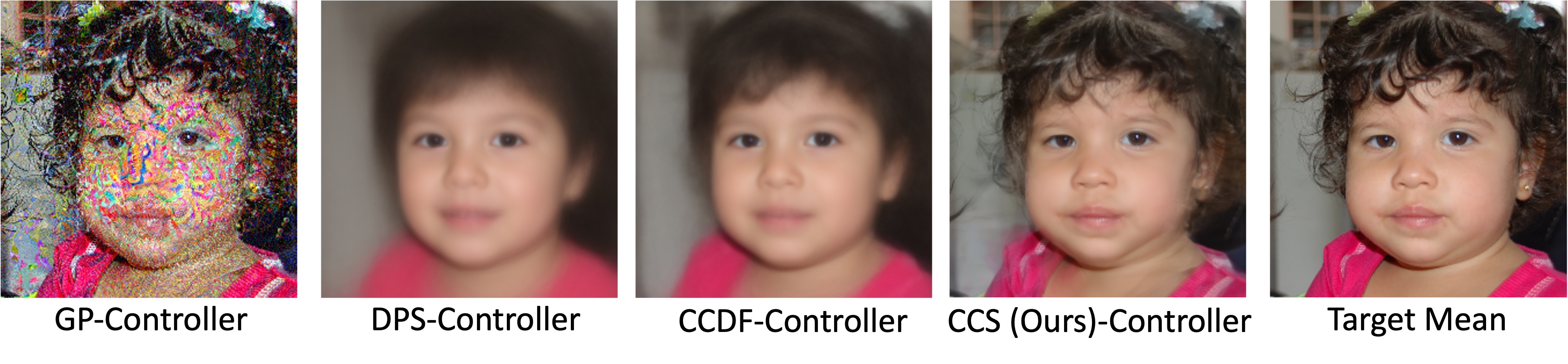}
    \vspace{-25px}
    \caption{We sample 120 images with a fixed target mean using different methods and analyze their sample mean (average pixel intensity). Our observations show that the sample mean of our method closely matches that of the original image.}
    \vspace{-5px}
    \label{fig:ccs_mean}
\end{figure}

\begin{figure}
    \centering
    \includegraphics[width=0.6\linewidth]{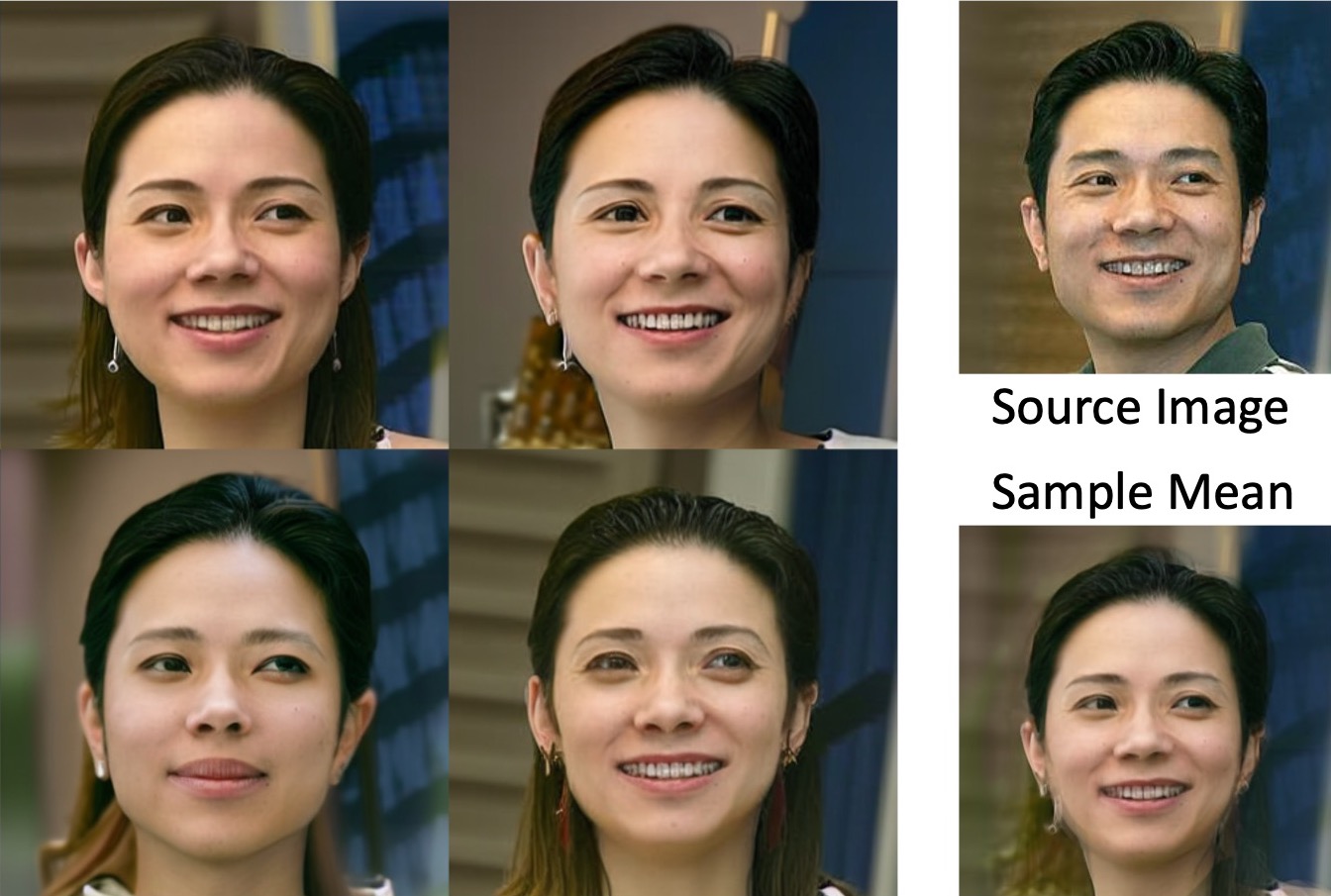}
    \caption{Example Samples of P-CCS around an Edited Mean.}
    \label{fig:edited image}
    \vspace{-15px}
\end{figure}

\textbf{Experimental setting.} For pixel diffusion models, we use the first 50 images from the validation data from the FFHQ-256 \cite{dps} dataset. Then we set each image as the target mean and then sample 120 images (6000 images in total) with each target mean with a target rMSE (square root of average L-2 norm of the residuals between the sample and target mean) of 0.12.  Then we test on the CIFAR-10 dataset. We randomly sample 20 images serving as target means, and then sample 120 images for each target mean with a target rMSE level of 0.11.

For Stable Diffusion, we use the SD1.5 checkpoint \cite{ldm}. We study a more challenging scenario (degraded low-resolution input images with conditional text-guided latent diffusion model). We sample 50 images from the validation set from Celeba-HQ dataset with resolution $256 \times 256$, and then use bicubic upsampling to upscale it to $512 \times 512$. Note that SD1.5 is not trained on the Celeba-HQ dataset so this demonstrates the generalization capability of algorithms. We use the same prompt and CFG level in the linearity control experiments. 

\textbf{Implementation details.} 
We follow Alg.~\ref{alg: CCS Full inverstion} in implementing our methods for pixel diffusion models, and Alg.~\ref{alg: CCS partial inverstion} for latent diffusion models. We take the pretrained models for FFHQ and CIFAR-10 from the improved/guided diffusion repos \cite{improvediffusion, beatGAN} for the pixel diffusion experiments, and the Stable Diffusion 1.5 \cite{ldm} for latent diffusion experiments. For LDMs, we set $t_0 = 45$, where $T = 50$ due to DDIM inversion performing worse with classifier-free guidance than unconditional models. We set the rMSE target to be 0.12, 0.11 for FFHQ and CIFAR-10 respectively, and 0.07 for Stable Diffusion experiments to test diverse control targets. The tolerance is set to be 0.01 in all cases. More details in the Appendix.\\

\textbf{Baselines.}
Since we target task is quite novel task without existing method may serve as exact baseline to compare with, we have to adapt some previous methods with our proposed controller algorithm as an add-on to this new setting.
\begin{itemize}[noitemsep, topsep=0pt]
    \item Gaussian Perturbation with Controller (GP-C): We add a Gaussian perturbation to the initial noisy image $x_{t0}$, where the perturbation scale is determined by our controller. This method resembles works that perform local editing \cite{loco}.
    \item (Latent) Diffusion Posterior Sampling \cite{dps, resample} with controller (DPS-C): We perform posterior sampling with $x_0$ as the measurement. The scale of the gradient term in (L)DPS can control the randomness, so we design a controller based on this. Details in the Appendix.
    \item ILVR with controller (ILVR-C): the ILVR algorithm \cite{ilvr} is for sampling high quality images based on a reference image. The larger the downsampling parameter gives a better diversity, we dynamically adjust that parameter as by our controller algorithm. Since it is designed only for DDPM, we do not experiment it with LDMs. Details in the Appendix.
    \item Come-closer-diffuse-faster with controller (CCDF-C): CCDF use DDPM forward to find a starting noise at $t_0$, and then perform reverse sampling based on that noise \cite{ccdf}. We adjust $t_0$ based on our controller algorithm.
\end{itemize}

\textbf{Evaluation metrics.}
We first compute pixel-wise metrics to validate our hypothesis that sample mean is close to the target mean.
\begin{itemize}[noitemsep, topsep=0pt]
\setlength{\itemsep}{0pt}
\item PSNR (Peak Signal-to-Noise Ratio): quantifies the pixel-wise difference between the target mean and the sample mean.
\item SD: the average of standard deviations of pixel intensities for each sampled image, which is used to measure the diversity of images.
\end{itemize}
Then we compute perceptual and reference-free metrics to measure the sample quality:
\begin{itemize}[noitemsep, topsep=0pt] 
    \item MUSIQ \cite{musiq}: measures the perceptual image quality, which focuses on low-level perceptual quality and is sensitive to blurs/noise/other distortions
    \item CLIP-IQA \cite{clipiqa}: measures the semantic image quality, which is more higher-level than MUSIQ
    \item Inception Score (IS) \cite{IS}: is used in the CIFAR-10 dataset to further measure image quality and diversity. Since CIFAR-10 has a low resolution and images are blurry, we report IS score instead of MUSIQ and CLIP-IQA for CIFAR-10.
\end{itemize}

\textbf{Results.} Quantitative results are demonstrated in Table.~\ref{tab:method_comparison1}, \ref{tab:method_comparison2}, \ref{tab:method_comparison3}. We observe that our CCS sampling method significantly outperforms all other methods in centering at a target mean when fixing the MSE level, while surprisingly maintaining superior image perceptual quality and diversity. Posterior sampling methods suffer from image quality degradation and diversity decreases. Qualitatively, we observe that the sample means of other methods look blurry or noisy, as demonstrated in Fig.~\ref{fig:ccs_mean}. Note that although GP satisfies $\mathbb{E}[\Delta x] = 0$, it pushes $\x_T'$ further from the spherical surface, where the diffusion model is not trained on. Empirically we observe very noisy samples and sample mean. 

\begin{table}[h!]
\centering
\scalebox{0.85}{
\begin{tabular}{|l|c|c|c|c |}
\hline 
 \textbf{Methods}  &  PSNR$\uparrow$&  SD$ \uparrow$&  CLIP-IQA$\uparrow$&  MUSIQ$\uparrow$\\ \hline
 GP-C      &  18.88                   &  0.028                       &  0.701       &  45.88     \\ \hline

 ILVR-C     & 20.04 & 0.070 & \underline{0.746} & 62.45\\

 DPS-C         &  21.02                   &  0.069                       &  0.738    &  64.60        \\ \hline
 CCDF-C               &  \underline{23.52}                   &  \underline{0.088}                       &  \underline{0.746}            &  \underline{66.15} \\ \hline
 CCS (Ours)-C        &  \textbf{25.13}                   &  \textbf{0.104}                       &  \textbf{0.750}        &  \textbf{66.79}    \\ \hline
\end{tabular}
}
\caption{Results of Pixel Diffusion models on the FFHQ Dataset with rMSE target 0.12.}
\label{tab:method_comparison1}
\end{table}

\begin{table}[h!]
\centering
\scalebox{0.85}{
\begin{tabular}{|l|c|c|c|}
\hline
\textbf{Methods}           &  PSNR$\uparrow$ &  SD$\uparrow$ &  IS$\uparrow$ \\ \hline
GP-C      & 24.66                   & 0.100                       & 7.56       \\ \hline
DPS-C         & 23.13                   & 0.054                       & 7.86       \\ \hline
CCDF-C               & 24.63                   & 0.099                       & 7.91             \\ \hline
CCS (Ours)-C        & \textbf{26.05}                   & \textbf{0.107}                       & \textbf{8.09}     \\ \hline
\end{tabular}}
\caption{Results of Pixel Diffusion models on the CIFAR-10 Dataset with rMSE target 0.11.}
\label{tab:method_comparison2}
\end{table}

\begin{table}[h!]
\centering
\scalebox{0.85}{
\begin{tabular}{|l|c|c|c|c |}
\hline
 \textbf{Methods}           &  PSNR$\uparrow$ &  SD$\uparrow$ &  CLIP-IQA$\uparrow$ &  MUSIQ$\uparrow$ \\ \hline
GP-C      & 23.02                   & 0.045                       & 0.721       & 48.91     \\ \hline
LDPS-C         & 24.56                  & 0.034                       & 0.721    & 29.07        \\ \hline
CCDF-C               & 27.66                   & 0.051                       & \textbf{0.735}            & 49.29 \\ \hline
CCS (Ours)-C        & \textbf{30.29}                   & \textbf{0.053}                       & 0.732        & \textbf{49.66}    \\ \hline
\end{tabular}
}
\caption{Results of the Stable Diffusion 1.5 on the Celeba-HQ Dataset with rMSE target 0.07.}
\label{tab:method_comparison3}
\end{table}
\vspace{-10px}
\subsection{Controllable Sampling for Image Editing Task}

We perform additional experiments with the conditional sampling for image editing using our CCS algorithm. We apply Algo.~\ref{alg: CCS partial inverstion} using the source prompt at DDIM inversion and the target prompt at reverse sampling. We demonstrate that we can sample diverse edited images as demonstrated in Fig.~\ref{fig:edited image}, where the diversity can be controlled by $C_0$.

\vspace{-10px}

\section{Conclusion}

In this work, we study a new problem: how to sample images with target statistical properties. We present a novel sampling algorithm and a novel controller method for diffusion models to sample with desired statistical properties. We also unveil an interesting linear response to perturbation phenomenon both theoretically and empirically. Extensive experiments show that our proposed method samples the closest to the target mean when controlling the MSE compared to other methods, while maintaining superior image quality and diversity.


\bibliography{example_paper}
\bibliographystyle{icml2025}

\newpage
\appendix
\onecolumn
\section{Proofs}
\subsection{Proof in Section \ref{subsec:linear, DDIM}}
\begin{proof}[Proof of Proposition \ref{prop: linearity, DDIM}]

Let $L_t(\x) := \eta_t \x + \lambda_t \nabla_{\xt}\log p_{t}(\x)$ be the one-step recursion.  Our $\x_0(a,t)$ is formally defined as $L_1\circ L_2 \circ \ldots \circ L_T(a)$.

    The second-order differentiability of $p_t$  implies the score function $\nabla \log p_t$ is first-order differentiable. 
Let $H_t$ be the Hessian matrix of $\log p_t$ ($H_{t}^{i,j} := \partial^2 \log p_t/\partial_i\partial_j$).
We have
\[
\nabla \log p_t (\x) = \nabla \log p_t (\w)  + H_t(\w)(\x -\w) + o(\lVert \x - \w \rVert_2).
\]

Therefore, for any fixed direction $\w$ and $\delta \in \mathbb R$,
\begin{align*}
    L_T(\x + \delta \w) &= \eta_T (\x + \delta \w) + \lambda_T \nabla_{\x}\log p_{T}(\x + \delta \w)\\
    & = \eta_T \x + \lambda_T \nabla_{\x}\log p_{T}(\x) + \lambda_T \delta H_T(\x) \w + \delta\eta_T \w+ o(\delta)\\
    & = L_T(\x) + \delta(\eta_T + \lambda_T H_T(\x))\w + o(\delta)\\
    & = L_T(\x) + \delta\gamma_T(\x) \w + o(\delta)
\end{align*}
where $\gamma_T(\x)$ is defined as
\[
\gamma_T(\x) = \eta_T + \lambda_T H_T(\x),
\]
is a matrix-valued function.

Applying $L_{T-1}$ on both sides of the above formula:
\begin{align*}
   L_{T-1} \circ L_T(\x + \delta \w) &= L_{T-1}\circ( L_T(\x) + \delta\gamma_T(\x) \w + o(\delta))\\
   & = \eta_{T-1}L_{T}(\x) + \delta\eta_{T-1}\gamma_T(\x) \w
   + o(\delta) + \lambda_{T-1}\nabla \log p_{T-1}\bigg(L_T(\x) + \delta\gamma_T(\x) \w + o(\delta))\bigg)\\
   &  = 
   \underbrace{\eta_{T-1} L_T(\x) + \lambda_{T-1}\nabla \log p_{T-1}(L_T(\x))}_\text{recursion on the unperturbed data $\x$} + 
   \underbrace{\delta\eta_{T-1}\gamma_T(\x)\w + 
   \delta\lambda_{T-1}H_{T-1}(L_T(\x))\gamma_T(\x)\w }_\text{linear term} \\
   &\qquad+ \underbrace{o(\delta)}_\text{lower order term}\\
   & = L_{T-1}\circ L_T(\x) + \delta\gamma_{T-1}(\x)\w + o(\delta).
\end{align*}
where 
\[
\gamma_{T-1}(\x) := \left(\eta_{T-1} I + \lambda_{T-1}H_{T-1}\left(L_T(\x)\right)\right) \gamma_T(\x)
\]

We could continue applying $L_{T-2}, L_{T-3} \ldots, L_1$ on the above formula, and conclude:
\begin{align}
    \x_0(\x_T + \lambda\Delta\x, T) = x_0(\x_T) + \lambda\gamma_0(\x_T)\Delta\x + o(\lambda). 
\end{align}
We might be particularly interested in the distance $\lVert  \x_0(\x_T + \lambda\Delta\x, T) -  \x_0(\x_T, T)\rVert$, our calculation directly implies:
\begin{align}
    \lVert  \x_0(\x_T + \lambda\Delta\x, T) -  \x_0(\x_T, T)\rVert_2 = \nonumber \\ \lVert \lambda\gamma_0(\x_T)\Delta\x \rVert_2 + o(\lambda) = \lambda \lVert\gamma_0(\x_T)\Delta\x \rVert_2 + o(\lambda). 
\end{align}
\end{proof}

\subsection{Proof in Section \ref{subsec: linear, ODE}}\label{subsec: proof of linear, ODE}

We first state the detailed assumptions posed in Section \ref{subsec: linear, ODE}. Define the function
\[
h(t, \y) := -\frac12\sqrt{\alpha(t)}\alpha'(t)\nabla\log p_t\left(\frac{\y}{\sqrt{\sigma^2(t) + 1}}\right).
\]
We assume this function has a continuous derivative (i.e., $C^1$) on the whole space $[0,T]\times \mathbb R^m$. Moreover, we assume there exists  $C(t)$ such that:
\[
\lVert h(t, \y) - h(t, \x) \rVert_2 \leq C(t) \lVert \y-\x\rVert_2,
\]
for every $\x, \y, t$, and $\max_{t\in[0,T]} C(t) \leq C <\infty$.

\begin{proof}[Proof of Proposition \ref{prop: linearity ODE}]
    We first show the ODE \eqref{eqn:idealized ODE} exists a unique solution. We can rewrite the \eqref{eqn:idealized ODE} as:
    \begin{align*}
    \diff \bar{\x}_t &= -\sqrt{1-\alpha(t)}\nabla\log p_t\left(\frac{\bar{\x}_t}{\sqrt{\sigma^2(t) + 1}}\right)\diff \sigma(t)\\
    & = -\sigma'(t)\sqrt{1-\alpha(t)}\nabla\log p_t\left(\frac{\bar{\x}_t}{\sqrt{\sigma^2(t) + 1}}\right)\diff t\\
    & = -\frac12\sqrt{\alpha(t)}\alpha'(t) \nabla\log p_t\left(\frac{\bar{\x}_t}{\sqrt{\sigma^2(t) + 1}}\right)\diff t \qquad \text{as} ~~\sigma(t) = \sqrt{(1- \alpha(t))/\alpha(t)}\\
    & = h(t, \bar\x_t) \diff t.
\end{align*}

Given $h(t,\y)\in C^1$ and unifomrly Lipschitz in $\y$, it follows from the Picard-Lindelöf Theorem (e.g., Theorem 1.1 of \cite{hartman2002ordinary}) that our ODE \eqref{eqn:idealized ODE} has a unique solution for any initialization $\bar\x_T = \bar \x$.

Next, it follows from Theorem 3.1 of \cite{hartman2002ordinary} that the solution $\bar\x_0(\bar\x, T) \in C^1$, i.e., the solution depends continuously and differentiably on its initialization \(\bar\x\). Thus,
\begin{align*}
     \bar\x_0(\x_T +\lambda \Delta\x, T) = \bar\x_0(\x_T) + \lambda J_{\bar\x}(\x_T) \Delta\x + o(\lambda),
\end{align*}
where $J_\x$ is the Jabobian matrix of the function $\bar\x_0(\bar\x, t)$ with respect to $\bar \x$. This concludes the proof of Proposition \ref{prop: linearity ODE}.
\end{proof}

\begin{proof}[Proof of Proposition \ref{prop: at most linear ODE}]
Let \(\bar\x_T\) and \(\bar\x_T + \lambda \Delta\x\) be two fixed initializations. Define  
\[
\y_t := \bar\x_t(\bar\x_T) - \bar\x_t(\bar\x_T + \lambda \Delta\x)
\]  
as the difference between the solutions of \eqref{eqn:idealized ODE} at time \(t \in [0, T]\). 

Taking derivative on $\y$ with respect to $t$ yields:
\begin{align*}
    \y'_t = h(t,\bar\x_t(\bar\x_T)) - h(t, \bar\x_t(\bar\x_T +  \lambda \Delta\x) ).
\end{align*}
By the Lipschitz continuity:
\begin{align*}
    \lVert \y'_t \rVert_2 \leq C \lVert \bar\x_t(\bar\x_T) - \bar\x_t(\bar\x_T + \lambda \Delta\x)\rVert_2 = C(t) \lVert \y_t \rVert_2
\end{align*}

Denote $\y_t$ by $(\y_{1,t},\y_{2,t}, \ldots, \y_{m,t})^\top$, we have:
\begin{align*}
    \frac{\diff \lVert \y_t\rVert_2}{\diff t} &=  \frac{\diff \sqrt{\sum_{i=1}^m \y_{i,t}^2}}{\diff t} \\
    & = \frac12 \frac{\sum_{i=1}^m 2 \y_{i,t} \y'_{i,t}}{\sqrt{\sum_{i=1}^m \y_{i,t}^2}}\\
    & = \frac{\sum_{i=1}^m  \y_{i,t} \y'_{i,t}}{\lVert \y_t\rVert_2}\\
    &\leq \frac{\lVert \y_t\rVert_2 \lVert \y'_t\rVert_2}{\lVert \y_t\rVert_2} \qquad\qquad \text{Cauchy-Schwarz inequality}\\
    & = \lVert \y'_t\rVert_2.
\end{align*}
Thererfore, we have
\begin{align*}
     \frac{\diff \lVert \y_t\rVert_2}{\diff t} \leq C(t) \lVert \y_t\rVert_2.
\end{align*}
Applying Grönwall's inequality on the function $\lVert \y_t\rVert$, we have:
\begin{align*}
    \lVert \y_t\rVert_2 \leq \exp\left(\int_t^TC(t)\diff t\right) \lVert \lambda\Delta\x \rVert_2
\end{align*}
for every $0\leq t\leq T$. Taking $t = 0$, we have
 \begin{align*}
        \lVert \bar\x_0(\x_T +\lambda \Delta\x, T) - \bar\x_0(\x_T) \rVert_2 \leq \exp\left(\int_0^TC(t)\diff t\right)\lvert \lambda \rvert \lVert \Delta \x \rVert_2.
    \end{align*}
    as claimed in Proposition \ref{prop: at most linear ODE}.
    \end{proof}
\subsection{Proof in Section \ref{sec: sampling with control}}
\begin{proof}[Proof of Proposition \ref{prop: centering feasibility}]
    It is known 
    \begin{align*}
          \bE[\lVert \x + \Delta\x \rVert_2^2] &= \lVert\bE[ \x + \Delta\x ]\rVert_2^2 + \tr(\Cov[\x + \Delta\x])\\
          & = \lVert \x \rVert_2^2 + \tr(\Cov[\Delta\x])\\
          &\geq \lVert \x \rVert_2^2
    \end{align*}
The equality is taken if and only if $\tr(\Cov[\Delta\x]) = \sum_{i} \Var[\Delta\x_i] = 0$. This is equivalent to saying that all components of $\Delta\x$ are  deterministic. Therefore, almost surely, $\Delta\x = \bE[\Delta\x] = 0$.  
\end{proof}

\begin{proof}[Proof of Proposition \ref{prop: CCS control initial distance}]
    Given a standard normal vector $\epsilon$, we claim the following holds:
    \begin{align*}
        \frac{\lVert \eps - \x_T\rVert_2^2}{d} &= \frac{\lVert \eps \rVert_2^2}{d} + \frac{\lVert \x_T\rVert_2^2 }{d} + \frac{-2\epsilon\cdot \x_T}{d}
        \rightarrow 2
    \end{align*}
    in $L^2$ as $d\rightarrow \infty$. To see this, notice the first term is
    \begin{align*}
        \frac{\sum_{i=1}^d\epsilon_i^2}{d}
    \end{align*}
    which converges to $1$ by the law of large numbers, since $\bE[\epsilon_i^2] = 1$. The second term converges to $1$ by our assumption. The last term converges to $0$ in $L^2$ as
    \begin{align*}
        \bE\left[\left\lVert\frac{-2\epsilon\cdot \x_T}{d}\right\rVert^2\right] = \frac{4\bE[\sum_{i}\x_{T,i}^2\bE[\epsilon_i^2]]}{d^2} = \frac{4(d+o(d))}{d^2} \rightarrow 0.
    \end{align*}
Therefore the distance $\lVert \eps - \x_T\rVert_2$ converges to $2\sqrt d$ as $d\rightarrow\infty$. Similarly we can show $\theta(\epsilon,\x_T)$, the angle between $\epsilon$ and $\x_T$ converges to $\pi/2$ as $d\rightarrow \infty$. In other words, $\epsilon$ is approximately orthogonal to $\x_T$ when the dimension $d$ is large. 

Therefore, with probability $1-o(1)$, the angle $\theta$ in Algorithm \ref{alg: CCS Full inverstion} is $\pi/2 \pm o(1)$, and $\lVert \eps - \x_T\rVert_2/2\sqrt d = 1 \pm o(1)$ as $d\rightarrow\infty$. Fix any $M \leq (2-\delta)\sqrt d$, since the spherical interpolation smoothly interpolate between $\x_0$ and $\eps$, there exists a $C$ satisfying 
Algorithm \ref{alg: CCS Full inverstion} with input $C$ output $\x_T'$ with distance $M$ to $\x_T$ with probability $1-o(1)$.  

We can indeed find an explicit $C_0$ with slightly weaker guarantees, set 
\[
C_0  = \cos^{-1}\left(1- \frac{M^2}{2\lVert\x_t \rVert_2^2}\right).
\]
Then with probability $1-o(1)$, $C_0 \in (0,\pi/2)$, and
\begin{align*}
  \left \lVert \frac{\sin(C_0)}{\sin(\theta)} \cdot \epsilon + \frac{\sin(\theta - C_0)}{\sin(\theta)} \cdot \x_T - \x_T\right\rVert \leq \left \lVert \frac{\sin(C_0)}{\sin(\theta)} \cdot \epsilon + \frac{\sin(\theta - C_0)}{\sin(\theta)} \cdot \x_T - \sin(C_0)\epsilon - \sin(\theta-C_0)\x_T\right\rVert\\ + \left \lVert \sin(C_0)\epsilon + \sin(\theta-C_0)\x_T - \x_T \right\rVert\\
\end{align*}
by triangle's inequality. Meanwhile, the first term is $o(\sqrt{d})$ as $\sin(\theta) = \sin(\pi/2 +o(1)) = 1+o(1)$ and $\cos(\theta) = 1-o(1)$. The square of the second term is
\begin{align*}
    \left \lVert \sin(C_0)\epsilon + \sin(\theta-C_0)\x_T - \x_T \right\rVert^2 &= \sin(C_0)^2 \lVert \epsilon \rVert^2 + (1 - \sin(\theta-C_0))^2 \lVert \x_T\rVert^2 + 2\sin(C_0)(\sin(\theta-C_0)-1) \epsilon\cdot \x_T\\
    & = \sin(C_0)^2 (d+o(1)) + (1 - \sin(\theta-C_0))^2 (d+o(1)) + o(d)\\
\end{align*}
The last term is $o(d)$ as $\epsilon\cdot\x_T/d \rightarrow 0$ as we analyzed above. Using again $\theta =\pi/2 + o(1)$, we know $\sin(\theta-C_0) = \sin(\pi/2-C_0) + o(1) = \cos(C_0) + o(1)$.
Hence we clean the above equation: 
\begin{align*}
     \left \lVert \sin(C_0)\epsilon + \sin(\theta-C_0)\x_T - \x_T \right\rVert^2 &=d( \sin(C_0)^2 + (1-\cos(C_0))^2) + o(d)\\
     & = d(2 - 2\cos(C_0)) + o(d)\\
     & = d\left(2 - 2 + \frac{M^2}{\lVert\xt\rVert_2^2}\right) + o(d)\\
     &  = M^2 + o(d),
\end{align*}
where the last equality follows from $\lVert \x_T\rVert_2^2 = d + o(1)$. Finally, taking the square root and plugging back into the triangle inequality, we have:
\[
 \left \lVert \x_T'- \x_T\right\rVert = M + o(\sqrt{d}).
\]
\end{proof}
\section{Additional Results and Experiments}

In this section, we clarify some implementation details, providing more details on algorithms and visualization. 
\subsection{More Implementation Details
}
For the P-CCS algorithm, In the Stable Diffusion experiments, we found that increasing $t_0$ to $T$ causes PSNR to drop. The reason is DDIM inversion is inexact in a CFG setting. However, by decreasing $t_0$ we increases the PSNR, but loss in diversity. We find at $t_0$ in the range of 40 to 48 is an ideal range, with $T = 50$, in which we choose $t_0 = 45$ based on tuning on one validation target mean image, which we exclude in our benchmarking.

For all methods, we set the batch size for controller tuning to be 24, and max number of iterations for tuning to be 6. Due to the linearity property of our proposed perturbation method, our controller can converge in fewer than 4 iterations in most cases. We observe DPS require significant more iterations, since the effect of changing scale has a very non-linear effect on the distance.
\begin{itemize}
    \item CCDF: Similar to \cite{ccdf}'s approach, we use DDPM forward on the target mean with a $t_0$, such that $x_t = \sqrt{\alpha_{t_0}} x_0 + \sqrt{1 - \alpha_t}\epsilon$, where $\epsilon$ is standard Gaussian noise. Then we run DDIM sampling start from $x_t$ and $t_0$. The controller algorithm automatically tune $t_0$ since a larger $t_0$ gives more diverse outputs. By this trend, we are able to perform binary search using our controller. Nevertheless, the number of timesteps is discrete, and limit controlling in a continuous range, also it is highly nonlinear. 
    \item DPS: We use the original codebase from DPS \cite{dps}, the measurement function we use is an identical mapping without noise. The samping step is given by $\mathbf{x}_{i-1} \gets \mathbf{x}'_{i-1} - \zeta_i \nabla_{\mathbf{x}_i} \|\mathbf{y} - \mathcal{A}(\hat{\mathbf{x}}_0)\|_2^2$. It mentions that by increasing the gradient scale $\zeta_i$, we can adjust between the sample diversity and faithness to measurement. However, there is no upper-bound fo $\zeta_i$, which makes it difficult for binary search. So we empirically set $max(\zeta_i) = 0.5$, and run the controller algorithm.

    \item GP: We add a random Gaussian perturbation to the initial noise, the larger the scale the more deviation from target mean, so we design a controller based on this.
    \item ILVR: We use the original codebase from ILVR\cite{ilvr}, note that the larger the downsampling factor $N$, the more deviation from input image, so we use this as an controller, setting $N_{max} = 64$. 
\end{itemize}

\subsection{More Algorithms}

Algo.~\ref{alg: CCS partial inverstion} demonstrates using P-CCS on Stable Diffusion.

Algo.~\ref{alg: PCCS-edit} demonstrates using P-CCS for controllable image editing sampling.
\begin{algorithm}
\caption{(Partial Inversion) CCS Sampling with Stable Diffusion}\label{alg: CCS partial inverstion}

\textbf{Requires: } target mean $x_0$, perturbation scale $C_0$, inversion time steps $t_0$, Encoder $\mathcal{E}$ and Decoder $\mathcal{D}$, a prompt $c$.

\textbf{Step 0:} Compute $\z_0 = \mathcal{E}(x_0)$, then compute the DDIM inversion of $\z_0$, i.e. $\z_T = \text{DDIM}^{-1}(\z_0, 0, t
_0, c)$

\textbf{Step 1:} Compute the noise from $\z_{t_0}$, by $\epsilon_{t_0} = \z_{t_0} - \sqrt{\alpha_{t_0}} \cdot \z_0$

\textbf{Step 2:} Sample noise $\epsilon \sim \bN(0,1 -\alpha_t)$. Then compute \[\theta = \cos^{-1}\left(\frac{\epsilon \cdot \epsilon_{target}}{||\epsilon||_2||\epsilon_{target}||_2}\right)\]

\textbf{Step 3:} 

Compute $\epsilon_{t_0}'$ using spherical interpolation formula:
\[
\epsilon_{t_0}' = \frac{\sin(C_0)}{\sin(\theta)} \cdot \epsilon + \frac{\sin(\theta - C_0)}{\sin(\theta)} \cdot \epsilon_{t_0}
\]

\textbf{Step 4:}
Compute $\z_{t_0}' = \sqrt{\alpha_{t_0}}\cdot \z_0 + \epsilon_{t_0}'$

\textbf{Step 5:} Output sample $\x_0' = \mathcal{D}(\z_0') = D(\text{DDIM}(\z_{t_0}', t_0, 0, c))$
\end{algorithm}

\begin{algorithm}
\caption{CCS Image Editing Sampling with Stable Diffusion}\label{alg: PCCS-edit}

\textbf{Requires: } target mean $x_0$, perturbation scale $C_0$, inversion time steps $t_0$, Encoder $\mathcal{E}$ and Decoder $\mathcal{D}$, source prompt $c_s$, target prompt $c_t$.

\textbf{Step 0:} Compute $\z_0 = \mathcal{E}(x_0)$, then compute the DDIM inversion of $\z_0$, i.e. $\z_T = \text{DDIM}^{-1}(\z_0, 0, t
_0, c_s)$

\textbf{Step 1:} Compute the noise from $\z_{t_0}$, by $\epsilon_{t_0} = \z_{t_0} - \sqrt{\alpha_{t_0}} \cdot \z_0$

\textbf{Step 2:} Sample noise $\epsilon \sim \bN(0,1 -\alpha_t)$. Then compute \[\theta = \cos^{-1}\left(\frac{\epsilon \cdot \epsilon_{target}}{||\epsilon||_2||\epsilon_{target}||_2}\right)\]

\textbf{Step 3:} 

Compute $\epsilon_{t_0}'$ using spherical interpolation formula:
\[
\epsilon_{t_0}' = \frac{\sin(C_0)}{\sin(\theta)} \cdot \epsilon + \frac{\sin(\theta - C_0)}{\sin(\theta)} \cdot \epsilon_{t_0}
\]

\textbf{Step 4:}
Compute $\z_{t_0}' = \sqrt{\alpha_{t_0}}\cdot \z_0 + \epsilon_{t_0}'$

\textbf{Step 5:} Output sample $\x_0' = \mathcal{D}(\z_0') = D(\text{DDIM}(\z_{t_0}', t_0, 0, c_t))$
\end{algorithm}

\subsection{More Figures}

Fig.~\ref{fig:sdccs} demonstrates example of applying P-CCS with SD1.5 on the Celeba-HQ dataset, we demonstrate that our algorithm can work well on in-the-wild images which are very different from the training. 

Fig.~\ref{fig:linear1} demonstrates the linear trend for each target mean on the FFHQ dataset. 

Fig.~\ref{fig:image-edit},\ref{fig:image-edit2} demonstrates an example of image editing controlled sampling with Alg.~\ref{alg: PCCS-edit}

\begin{figure*}
    \centering
    \includegraphics[width=1.0\linewidth]{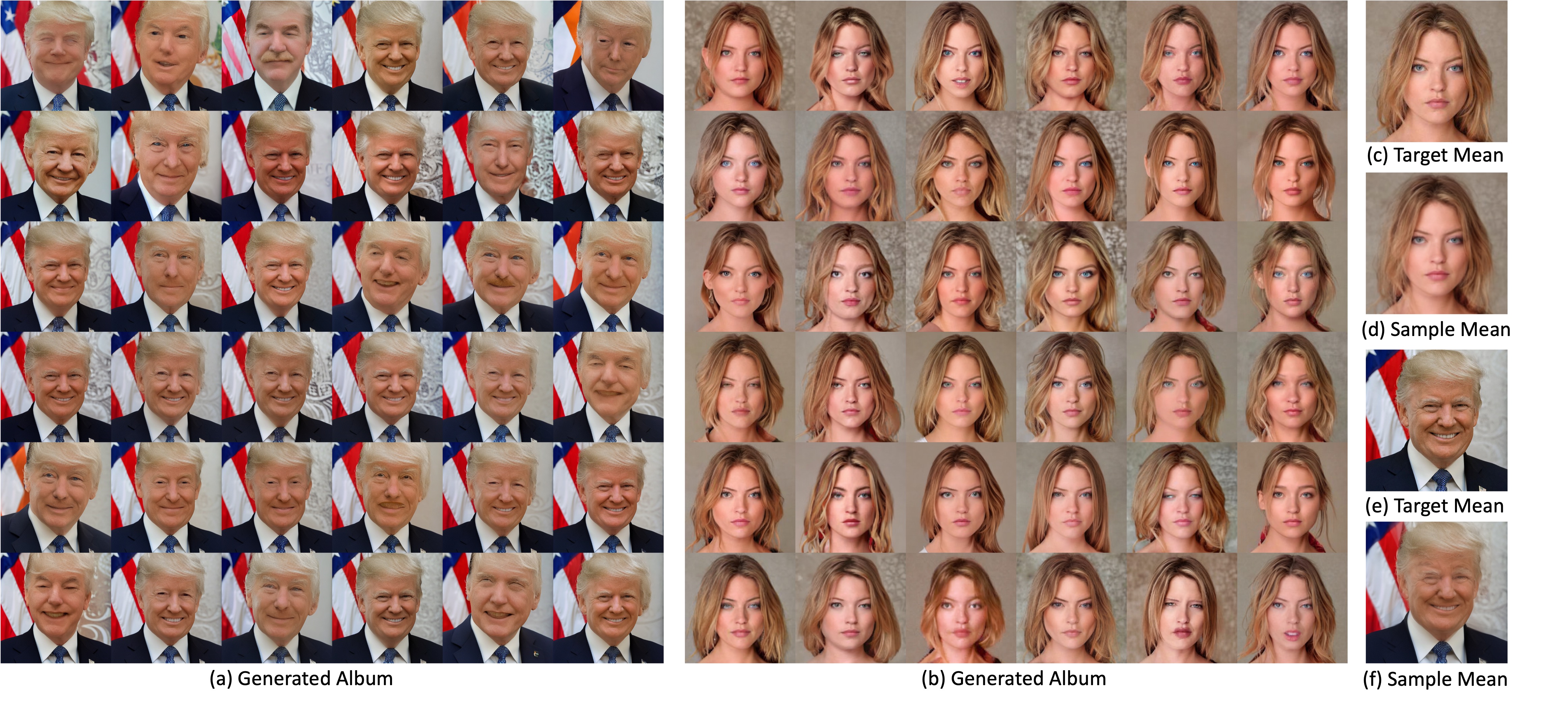}
    \caption{CCS-CT Sampled Images with Stable Diffusion 1.5. (a) Samples with an in-the-wild target mean (e) and a target rMSE 0.09; (b) Samples with a target mean (d) from Celeba-HQ dataset and a target rMSE 0.07; (d): sample mean from (a); (f): sample mean from (b). }
    \label{fig:sdccs}
\end{figure*}

\begin{figure*}
    \centering
    \includegraphics[width=1.0\linewidth]{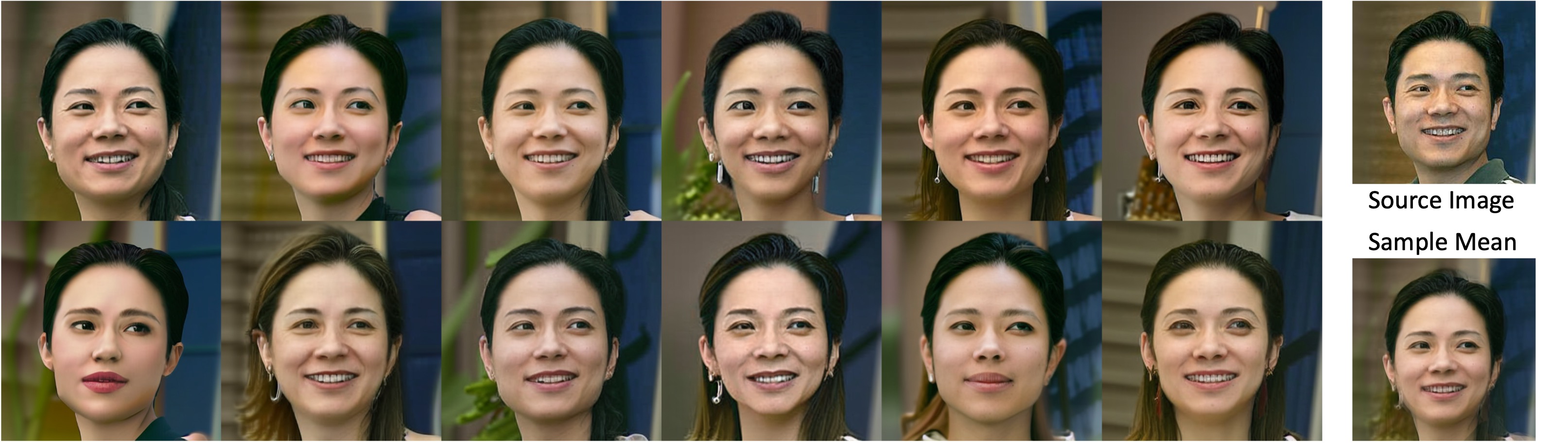}
    \caption{Image Editing Samples with Stable Diffusion 1.5, the source prompt is given by `a high-quality portrait of a man', and the target prompt is given by `a high-quality portrait of a woman', the target MSE level is given by 0.10}
    \label{fig:image-edit}
\end{figure*}

\begin{figure*}
    \centering
    \includegraphics[width=1.0\linewidth]{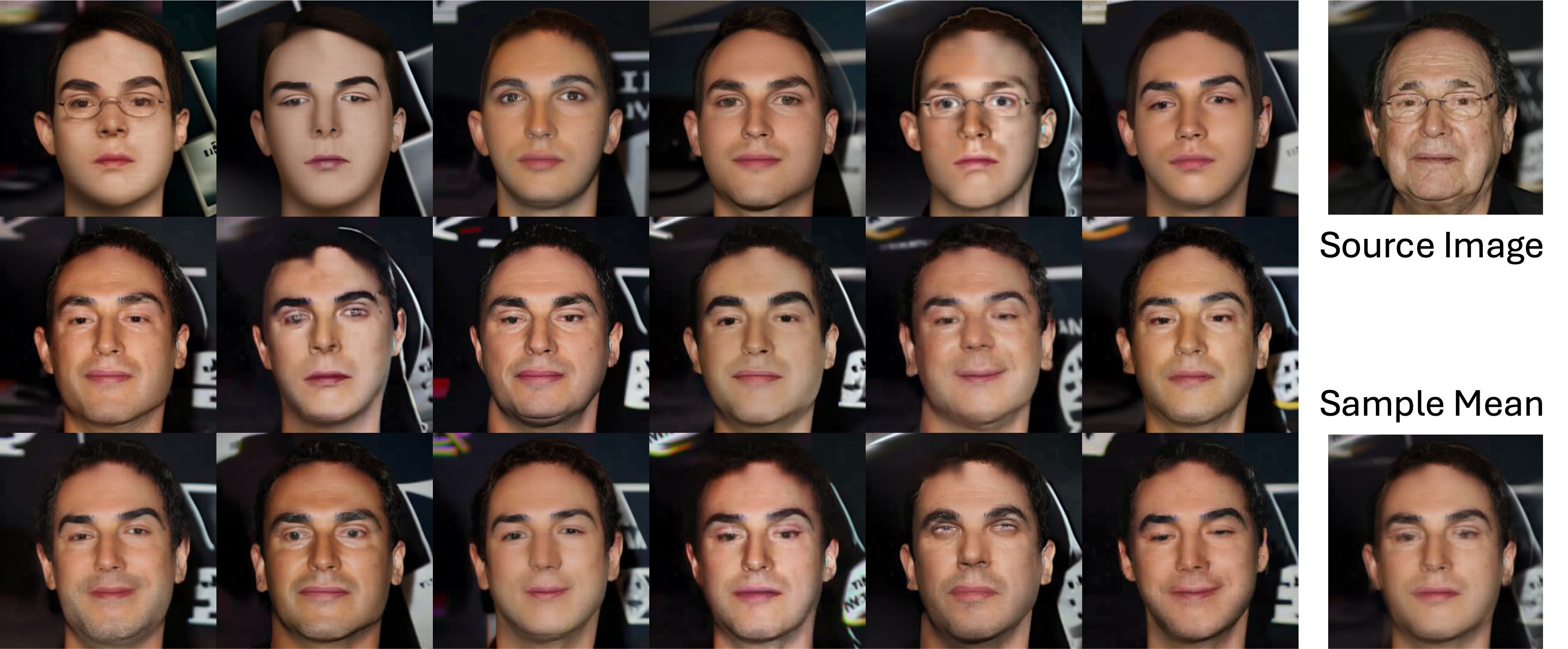}
    \caption{Image Editing Samples with Stable Diffusion 1.5, the source prompt is given by `a high-quality portrait of an old man', and the target prompt is given by `a high-quality portrait of a young man', the target MSE level is given by 0.09}
    \label{fig:image-edit2}
\end{figure*}

\begin{figure}
    \centering
    \includegraphics[width=0.7\linewidth]{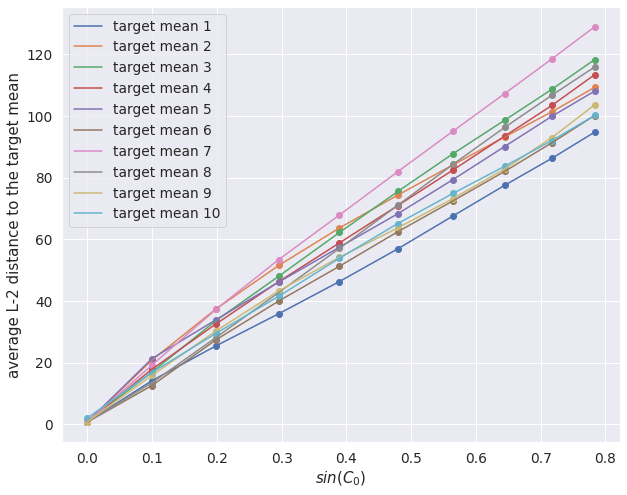}
    \caption{We first sample 50 images around each target mean in the FFHQ dataset. We then obtain the DDIM inverse of each target mean, and then add spherical perturbation to it. When the scale of perturbations $\sin(C_0)$ increases, the average of norms of the residuals between each sample and the target mean approximately increases linearly.}
    \label{fig:linear1}
\end{figure}


\end{document}